\documentclass{article}

\usepackage{graphicx}
\usepackage[hidelinks,linktocpage]{hyperref}  
\hypersetup{hypertexnames=false}
\usepackage{lmodern}
\usepackage[left=3.5cm,right=3.5cm,top=3.5cm,bottom=3.5cm]{geometry} 
\usepackage{booktabs}                      

\usepackage{amsthm}
\usepackage{amsmath}
\usepackage{amsfonts}
\usepackage{dsfont}                           
\usepackage{stix}                               
\usepackage{nicefrac}                      
\usepackage{enumerate}    
\usepackage{enumitem}

\usepackage{xcolor}

\usepackage{cite}

\usepackage[title]{appendix}

\makeatletter
\let\amstexbig\big
\def\newbig#1{%
  \ifx#1|%
    \expandafter\@firstoftwo
  \else
    \expandafter\@secondoftwo
  \fi
  {\big@bar}%
  {\amstexbig{#1}}%
}
\AtBeginDocument{\let\big\newbig}
\def\big@bar{\bBigg@{1.1}|}
\makeatother

\def\SX{\mathscr{X}}

\def\CD{\mathcal{D}}
\def\CE{\mathcal{E}}
\def\CG{\mathcal{G}}
\def\CH{\mathcal{H}}

\def\CO{\mathcal{O}}

\def\CS{\mathcal{S}}

\def\BE{\mathbb{E}}

\def\BI{\mathbb{I}}

\def\BN{\mathbb{N}}

\def\BR{\mathbb{R}}

\def\BfC{\mathbf{C}}
\def\BfK{\mathbf{K}}

\def\BfW{\mathbf{W}}
\def\BfX{\mathbf{X}}
\def\BfY{\mathbf{Y}}

\def\BfB{\mathbf{B}}

\def\rmd{\mathrm{d}}
\def\rml{\mathrm{l}}
\def\rmr{\mathrm{r}}

\def\smc{\boldsymbol{c}}

\def\Frob{\mathrm{F}}
\def\seco{\mathrm{s}}

\DeclareMathOperator{\vspan}{span}

\theoremstyle{plain}
\newtheorem{theorem}{Theorem}[section]

\theoremstyle{definition}
\newtheorem{remark}{Remark}[section]
\newcommand{\fin} {\hfill\hbox{$\triangleleft$}}

\title{Local optimisation of Nystr\"om samples\\
through stochastic gradient descent}

\author{Matthew \textsc{Hutchings}\footnotemark[2]\,  \footnotemark[4] \and Bertrand \textsc{Gauthier}\footnotemark[3]\,  \footnotemark[4]}

\begin{document}

\footnotetext[2]{HutchingsM1@cardiff.ac.uk}
\footnotetext[3]{GauthierB@cardiff.ac.uk}
\footnotetext[4]{Cardiff University, School of Mathematics\\ \hspace*{1.4em}\hspace*{1ex}Abacws, Senghennydd Road, Cardiff, CF24 4AG, United Kingdom}

\maketitle

\begin{abstract}
We study a relaxed version of the column-sampling problem for the Nystr\"om approximation of kernel matrices, where approximations are defined from multisets of landmark points in the ambient space; such multisets are referred to as Nystr\"om samples.
We consider an unweighted variation of the radial squared-kernel discrepancy (SKD) criterion as a surrogate for the classical criteria used to assess the Nystr\"om approximation accuracy; in this setting, we discuss how Nystr\"om samples can be efficiently optimised through stochastic gradient descent.
We perform numerical experiments which demonstrate that the local minimisation of the radial SKD yields
Nystr\"om samples with improved Nystr\"om approximation accuracy.
\end{abstract}
stochastic gradient descent.

\section{Introduction}
\label{sec:Intro}

In Data Science, the Nystr\"om method refers to a specific technique for the low-rank approximation of symmetric positive-semidefinite (SPSD) matrices; see e.g. \cite{drineas2005nystrom, kumar2012sampling, wang2016towards, gittens2016revisiting, derezinski2020improved}. Given an $N \times N$ SPSD matrix $\BfK$, with $N\in\BN$, the Nystr\"om method consists of selecting a sample of $n\in\BN$ columns of $\BfK$, generally with $n\ll N$, and next defining a low-rank approximation $\hat{\BfK}$ of $\BfK$ based on this sample of columns. More precisely, let $\smc_{1},\cdots,\smc_{N}\in\BR^{N}$ be the columns of $\BfK$, so that $\BfK=(\smc_{1}|\cdots|\smc_{N})$, and let $I=\{i_{1}, \cdots, i_{n}\} \subseteq \{1, \cdots, N\}$ denote the indices of a sample of $n$ columns of $\BfK$ (note that $I$ is a multiset, i.e. the indices of some columns might potentially be repeated). Let $\mathbf{C}=(\smc_{i_{1}}|\cdots|\smc_{i_{n}})$ be the $N \times n$ matrix defined from the considered sample of columns of $\BfK$, and let $\mathbf{W}$ be the $n \times n$ principal submatrix of $\BfK$ defined by the indices in $I$, i.e. the $k, l$ entry of $\BfW$ is $[\BfK]_{i_{k},i_{l}}$, the $i_{k}, i_{l}$ entry of $\BfK$. The Nystr\"om approximation of $\BfK$ defined from the sample of columns indexed by $I$ is given by
\begin{equation}
\label{eq:NystromApproxColumnSamp}  
\hat{\BfK} = \BfC\BfW^{\dag}\BfC^{T}, 
\end{equation}
with $\BfW^{\dag}$ the Moore-Penrose pseudoinverse of $\BfW$. The column-sampling problem for Nystr\"om approximation consists of designing samples of columns such that the induced approximations are as accurate as possible (see Section \ref{sec:Nystrom-accuracy} for more details).

\subsection{Kernel Matrix Approximation}
\label{sec:kernelMatApprox}

If the initial SPSD matrix $\BfK$ is a kernel matrix, defined from a SPSD kernel $K$ and a set or multiset of points $\CD=\{x_{1},\cdots,x_{N}\}\subseteq\SX$ (and with $\SX$ a general ambient space), i.e. the $i, j$ entry of $\BfK$ is $K(x_{i},x_{j})$, then a sample of columns of $\BfK$
is naturally associated with a subset of $\CD$; more precisely, a sample of columns $\{\smc_{i_{1}},\cdots,\smc_{i_{n}}\}$, indexed by $I$, naturally defines a multiset $\{x_{i_{1}},\cdots,x_{i_{n}}\}\subseteq\CD$, so that the induced Nystr\"om approximation can in this case be regarded as an approximation induced by a subset of points in $\CD$. Consequently, in the kernel-matrix framework, instead of relying only on subsets of columns, we may more generally consider Nystr\"om approximations defined from a multiset $\CS\subseteq\SX$. Using matrix notation, the Nystr\"om approximation of $\BfK$ defined by a subset $\CS=\{s_{1},\cdots,s_{n}\}$ is the $N\times N$ SPSD matrix $\hat{\BfK}(\CS)$, with $i, j$ entry
\begin{equation}
\label{eq:NystromApproxColumnSampKernel}  
\big[\hat{\BfK}(\CS)\big]_{i,j} = \mathbf{k}^{T}(x_{i})\mathbf{K}_{\mathcal{S}}^\dagger\mathbf{k}(x_j),
\end{equation}
where $\mathbf{K}_{\mathcal{S}}$ is the $n \times n$ kernel matrix defined by the kernel $K$ and the subset $\CS$, and where
\[
\mathbf{k}(x)=\big(K(x, s_{1}), \cdots, K(x, s_{n})\big)^T\in\mathbb{R}^{n}.
\]
We shall refer to such a set or multiset $\CS$ as a \textit{Nystr\"om sample}, and to the elements of $\CS$ as \textit{landmark points}; the notation $\hat{\BfK}(\CS)$ emphasises that the considered Nystr\"om approximation of $\BfK$ is induced by $\CS$. As in the column-sampling case, the landmark-point-based framework naturally raises questions related to the characterisation and the design of efficient Nystr\"om samples $\CS$ (i.e. leading to accurate approximations of $\BfK$). As an interesting feature, Nystr\"om samples of size $n$ may be regarded as elements of $\SX^{n}$, and if the underlying set $\SX$ is regular enough, they might be directly optimised on $\SX^{n}$; the situation we consider in this work corresponds to the case $\SX=\BR^{d}$, with $d\in\BN$, but $\SX$ may more generally be a differentiable manifold. 

\begin{remark}
\label{rem:ApproxRkhsSub}  
If we denote by $\CH$ the reproducing kernel Hilbert space (RKHS, see e.g. \cite{berlinet2004reproducing, paulsen2016introduction}) of real-valued functions on $\SX$ associated with $K$, we may then note that the matrix $\hat{\BfK}(\CS)$ is the kernel matrix defined by $K_{S}$ and the set $\CD$, with $K_{S}$ the reproducing kernel of the subspace 
 \[
 \CH_{S}=\vspan\{k_{s_1},\cdots,k_{s_{n}}\}\subseteq\CH, 
\]
where, for $t\in\SX$, the function $k_{t}\in\CH$ is defined as $k_{t}(x)=K(x,t)$, for all $x\in\SX$.
\fin
\end{remark}

\subsection{Assessing the Accuracy of Nystr\"om Approximations}
\label{sec:Nystrom-accuracy}

In the classical literature on the Nystr\"om approximation of SPSD matrices, the accuracy of the approximation induced by a Nystr\"om sample $\CS$ is often assessed through the following criteria:
\begin{enumerate}[label={(C.\arabic*)}]   
\itemsep0em
 \item  
 	$\big\|\BfK - \hat{\BfK}(\CS)\big\|_{*}$, with $\|.\|_{*}$ the trace norm; \label{it:CrtiTrace}
\item 
	$\big\|\BfK - \hat{\BfK}(\CS)\|_{\Frob}$,  with $\|.\|_{\Frob}$ the Frobenius norm; \label{it:CritFrob}
\item 
	$\big\|\BfK - \hat{\BfK}(\CS)\big\|_{2}$, with $\|.\|_{2}$ the spectral norm. \label{it:CritSpec}
\end{enumerate}
Although defining relevant and easily interpretable measures of the approximation error, these criteria are relatively costly to evaluate. Indeed, each of them involves the inversion or pseudoinversion of the kernel matrix $\BfK_{\CS}$, with complexity $\CO(n^{3})$. The evaluation of the criterion \ref{it:CrtiTrace} also involves the computation of the $N$ diagonal entries of $\hat{\BfK}(\CS)$, leading to an overall complexity of $\CO(n^{3}+Nn^{2})$. The evaluation of \ref{it:CritFrob} involves the full construction of the matrix $\hat{\BfK}(\CS)$, with an overall complexity of $\CO(n^{3}+n^{2}N^{2})$, and the evaluation of \ref{it:CritSpec} in addition requires the computation of the largest eigenvalue of an $N\times N$ SPSD matrix, leading to an overall complexity of $\CO(n^{3}+n^{2}N^{2}+N^{3})$. If $\SX=\BR^{d}$, then the evaluation of the partial derivatives of these criteria (regarded as maps from $\SX^{n}$ to $\BR$) with respect to a single coordinate of a landmark point has a complexity similar to the complexity of evaluating the criteria themselves. As a result, a direct optimisation of these criteria over $\SX^{n}$ is intractable in most practical applications.  

\subsection{Radial Squared-Kernel Discrepancy}
\label{sec:RadialSkd}

As a surrogate for the criteria \ref{it:CrtiTrace}-\ref{it:CritSpec}, and following the connections between the Nystr\"om approximation of SPSD matrices, the approximation of integral operators with SPSD kernels and the kernel embedding of measures,
we consider the following \textit{radial squared-kernel discrepancy} criterion (radial SKD, see \cite{Gauthier2018SKD, Gauthier2021SKD}), denoted by $R$ and given by, for $\CS=\{s_{1},\cdots,s_{n}\}$, 
\begin{equation}
\label{eq:RadialSKD}
R(\CS) =
\|\BfK\|_{\Frob}^{2} - \frac{1}{\|\BfK_{\CS}\|_{\Frob}^{2}} \bigg( \sum_{i=1}^{N} \sum_{j=1}^{n} K^{2}(x_i, s_j) \bigg)^{2},
\text{ if $\|\BfK_{\CS}\|_{\Frob}>0$}, 
\end{equation}
and $R(\CS) =\|\BfK\|_{\Frob}^{2}$ if $\|\BfK_{\CS}\|_{\Frob}=0$;  
the notation $K^{2}(x_{i}, s_{j})$ stands for $\big(K(x_{i}, s_{j})\big)^{2}$. We may note that $R(\CS)\geqslant 0$. 
In \eqref{eq:RadialSKD}, the evaluation of the term $\|\BfK\|_{\Frob}^{2}$ has complexity $\CO(N^{2})$; nevertheless, this term does not depend on the Nystr\"om sample $\CS$, and may thus be regarded as a constant. The complexity of the evaluation of the term $R(\mathcal{S})-\|\mathbf{K}\|_{\Frob}^{2}$, i.e. of the radial SKD up to the constant $\|\BfK\|_{\Frob}^{2}$, is $\CO(n^{2}+nN)$, and the same holds for the complexity of the evaluation of the partial derivative of $R(\CS)$ with respect to a coordinate of a landmark point, see equation \eqref{eq:radial-SKD-true-partial-pos} below. We may in particular note that the evaluation of the radial SKD criterion or its partial derivatives does not involve the inversion or pseudoinversion of the $n\times n$ matrix $\BfK_{\CS}$.

\begin{remark}
\label{rem:AboutRadialSKD}  
From a theoretical standpoint, the radial SKD criterion measures the distance, in the Hilbert space of all Hilbert-Schmidt operators on $\CH$, between the integral operator corresponding to the initial matrix $\BfK$, and the projection of this operator onto the subspace spanned by an integral operator defined from the kernel $K$ and a uniform measure on $\CS$. The radial SKD may also be defined for non-uniform measures, and the criterion in this case depends not only on $\CS$, but also on a set of relative weights associated with each landmark point in $\CS$; in this work, we only focus on the uniform-weight case.
See \cite{Gauthier2018SKD, Gauthier2021SKD} for more details. 
\fin
\end{remark}

The following inequalities hold: 
\[
\big\|\BfK-\hat{\BfK}(\CS)\big\|_{2}^{2}
\leqslant \big\|\BfK - \hat{\BfK}(\CS)\big\|_{\Frob}^{2}
\leqslant R(\CS)
\leqslant \|\BfK\|_{\Frob}^{2},
\quad\text{and}\quad
\frac{1}{N}\big\|\BfK - \hat{\BfK}(\CS)\big\|_{*}^{2}
\leqslant \big\|\BfK - \hat{\BfK}(\CS)\big\|_{\Frob}^{2},  
\]
which, in complement to the theoretical properties enjoyed by the radial SKD, further support the use of the radial SKD as a numerically affordable surrogate for \ref{it:CrtiTrace}-\ref{it:CritSpec} (see also the numerical experiments in Section \ref{sec:experiments}). 

From now on, we assume that $\SX=\BR^{d}$. Let $[s]_{l}$, with $l\in\{1,\cdots,d\}$, be the $l$-th coordinate of $s$ in the canonical basis of $\SX = \BR^{d}$. For $x\in\SX$, we denote by (assuming they exist)
\begin{equation}\label{eq:KernelLocPartialDeriv}
\partial^{[\rml]}_{[s]_{l}}K^{2}(s,x)   
\quad\text{ and }\quad
\partial^{[\rmd]}_{[s]_{l}}K^{2}(s,s)  
\end{equation}
the partial derivatives of the maps $s\mapsto K^{2}(s,x)$ and $s\mapsto K^2(s,s)$  at $s$ and with respect to the $l$-th coordinate of $s$, respectively;
the notation $\partial^{[\rml]}$ indicates that the left entry of the kernel is considered, while $\partial^{[\rmd]}$  refers to the diagonal of the kernel; we use similar notations for any kernel function on $\SX\times\SX$.

For a fixed number of landmark points $n\in\BN$, the radial SKD criterion can be regarded as a function from $\SX^{n}$ to $\BR$. For a Nystr\"om sample $\CS=\{s_{1},\cdots,s_{n}\}\in\SX^{n}$, and for $k \in \{1,\cdots, n\}$ and $l \in \{1,\cdots, d\}$, we denote by $\partial_{[s_{k}]_{l}} R(\CS)$ the partial derivative of the map $R:\SX^{n}\to\BR$ at $\CS$ with respect to the $l$-th coordinate of the $k$-th landmark point $s_{k}\in\SX$. We have
\begin{align}
\begin{split}\label{eq:radial-SKD-true-partial-pos}
\partial_{[s_{k}]_{l}} R(\CS)
& = \frac{1}{\|\BfK_{\CS}\|_{\Frob}^{4}}\bigg( \sum_{i=1}^{N}\sum_{j=1}^{n} K^{2}(s_{j}, x_{i}) \bigg)^{2}
\bigg( \partial_{[s_{k}]_{l}}^{[\rmd]}K^{2}(s_{k}, s_{k}) + 2 \sum_{\substack{j=1, \\ j \neq k}}^{n} \partial_{[s_{k}]_{l}}^{[\rml]} K^2(s_k, s_j) \bigg) \\
& \quad \quad \phantom{x} - \frac{2}{\|\BfK_{\CS}\|_{\Frob}^{2}} \bigg(\sum_{i=1}^{N} \sum_{j=1}^{n}K^2(s_j, x_i) \bigg)
\bigg( \sum_{i=1}^N \partial_{[s_{k}]_{l}}^{[\rml]} K^2(s_k, x_i) \bigg). 
\end{split}
\end{align}
In this work, we investigate the possibility to use the partial derivatives \eqref{eq:radial-SKD-true-partial-pos}, or stochastic approximations of these derivatives, to directly optimise the radial SKD criterion $R$ over $\SX^{n}$ via gradient or stochastic gradient descent; the stochastic approximation schemes we consider aim at reducing the burden of the numerical cost induced by the evaluation of the partial derivatives of $R$ when $N$ is large. 
 
The document is organised as follows. In Section~\ref{sec:ConvResult}, we discuss the convergence of a gradient descent with fixed step size for the minimisation of $R$ over $\SX^{n}$. The stochastic approximation of the gradient of the radial SKD criterion \eqref{eq:RadialSKD} is discussed in Section~\ref{sec:stoch-approx}, and some numerical experiments are carried out in Section~\ref{sec:experiments}. Section~\ref{sec:conclusion} consists of a concluding discussion, and the Appendix contains a proof of Theorem~\ref{thm:lipschitz}.

\section{A Convergence Result}
\label{sec:ConvResult}

We use the same notation as in Section~\ref{sec:RadialSkd} (in particular, we still assume that $\SX=\BR^{d}$), and by analogy with \eqref{eq:KernelLocPartialDeriv}, for $s$ and $x\in\SX$, and for $l\in\{1,\cdots,d\}$, we denote by
$\partial^{[\rmr]}_{[s]_{l}}K^{2}(x,s)$ the partial derivative of the map $s\mapsto K^{2}(x,s)$ with respect to the $l$-th coordinate of $s$. Also, for a fixed $n\in\BN$, we denote by $\nabla R(\CS)\in\SX^{n}=\mathbb{R}^{nd}$ the gradient of $R:\SX^{n}\to\BR$ at $\CS$; in matrix notation, we have
\[
\nabla R(\CS) = \Big(\big(\nabla_{s_1}R(\CS)\big)^{T}, \cdots, \big(\nabla_{s_1}R(\CS)\big)^{T}\Big)^{T}, 
\]
with $\nabla_{s_{k}} R(\CS) = \big(\partial_{[s_{k}]_{1}} R(\CS), \cdots, \partial_{[s_{k}]_{d}} R(\CS)\big)^{T}\in\mathbb{R}^{d}$ for $k\in\{1,\cdots,n\}$.

\begin{theorem}\label{thm:lipschitz}
We make the following assumptions on the squared-kernel $K^2$, which we assume hold for all $x$ and $y\in\SX=\BR^{d}$, and all $l$ and $l' \in\{1, \cdots, d\}$, uniformly:  
\begin{enumerate}[label={(C.\arabic*)}]
\itemsep0em  
\item there exists $\alpha > 0$ such that $K^2(x, x) \geqslant \alpha$; \label{itm:sq-kernel-diag-lower-bound} 
\item there exists $M_{1} > 0$ such that $\big|\partial_{[x]_l}^{[\rmd]} K^2(x, x)\big| \leqslant M_{1}$ and $\big|\partial_{[x]_l}^{[\rml]} K^2(x, y)\big| \leqslant M_{1}$; \label{itm:sq-kernel-partial-bound} 
\item there exists $M_{2} > 0$ such that $\big|\partial_{[x]_l}^{[\rmd]} \partial_{[x]_{l'}}^{[\rmd]} K^2(x, x)\big| \leqslant M_{2}$,  $\big|\partial_{[x]_l}^{[\rml]} \partial_{[x]_{l'}}^{[\rml]} K^2(x, y)\big| \leqslant M_2$ and\\
$\big|\partial_{[x]_l}^{[\rml]} \partial_{[y]_{l'}}^{[\rmr]} K^2(x, y)\big| \leqslant M_{2}$. 
\label{itm:sq-kernel-partial-2-bound}
\end{enumerate} 
Let $\CS$ and $\CS'\in\BR^{nd}$ be two Nystr\"om samples; under the above assumptions, there exists $L>0$ such that
\[
\big\|\nabla R(\CS)-\nabla R(\CS')\big\|
\leqslant L\big\|\CS-\CS'\big\| 
\]  
with $\|.\|$ the Euclidean norm of $\BR^{nd}$; in other words, the gradient of $R:\BR^{nd}\to\BR$ is Lipschitz-continuous with Lipschitz constant $L$. 
\end{theorem}

Since $R$ is bounded from below, for $0<\gamma\leqslant 1/L$ and independently of the considered initial Nystr\"om sample $\CS^{(0)}$, Theorem~\ref{thm:lipschitz} entails that a gradient descent from $\CS^{(0)}$, with fixed stepsize $\gamma$ for the minimisation of $R$ over $\SX^{n}$, produces a sequence of iterates that converges to a critical point of $R$. Barring some specific and largely pathological cases, the resulting critical point is likely to be a local minimum of $R$, see for instance \cite{lee2016gradient}. See the Appendix for a proof of Theorem~\ref{thm:lipschitz}. 

The conditions considered in Theorem~\ref{thm:lipschitz} ensure the existence of a general Lipschitz constant $L$ for the gradient of $R$; they, for instance, hold for all sufficiently regular Mat\'ern kernels (thus including the Gaussian or squared-exponential kernel). These conditions are only sufficient conditions for the convergence of a gradient descent for the minimisation of $R$. By introducing additional problem-dependent conditions, some convergence results might be obtained for more general squared kernels $K^{2}$ and adequate initial Nystr\"om samples $\CS^{(0)}$. For instance, the condition \ref{itm:sq-kernel-diag-lower-bound} simply aims at ensuring that $\|\BfK_{\CS}\|_{\Frob}^{2}\geqslant n\alpha>0$ for all $\CS\in\SX^{n}$; this condition might be relaxed to account for kernels with vanishing diagonal, but one might then need to introduce ad hoc conditions to ensure that $\|\BfK_{\CS}\|_{\Frob}^{2}$ remains large enough during the minimisation process.

\section{Stochastic Approximation of the Radial SKD Gradient}
\label{sec:stoch-approx}

The complexity of evaluating a partial derivative of $R:\SX^{n}\to\BR$ is $\CO(n^{2}+nN)$, which might become prohibitive for large values of $N$. 
To overcome this limitation, stochastic approximations of the gradient of $R$ might be considered (see e.g. \cite{bottou2018optimization}).

The evaluation of \eqref{eq:radial-SKD-true-partial-pos} involves, for instance, terms of the form
$\sum_{i=1}^{N}K^2(s, x_{i})$, with $s\in\SX$ and $\CD=\{x_{1},\cdots,x_{N}\}$. Introducing a random variable $X$ with uniform distribution on $\CD$, we can note that 
\[
\sum_{i=1}^N K^2(s, x_i) = N \BE\big[K^2(s, X)\big], 
\]
and the mean $\BE[K^2(s, X)]$ may then, classically, be approximated by random sampling. More precisely, if $X_{1},\cdots,X_{b}$ are $b\in\BN$
copies of $X$, we have 
\[
\BE\big[K^2(s, X)\big] = \frac{1}{b} \sum_{j=1}^b \BE\big[K^2(s, X_j)\big]
\quad \text{and} \quad
\BE\big[\partial_{[s]_l}^{[\rml]} K^2(s, X)\big]=\frac{1}{b}\sum_{j=1}^b \BE\big[\partial_{[s]_l}^{[\rml]} K^2(s, X_j)\big],  
\]
so that we can easily define unbiased estimators of the various terms appearing in \eqref{eq:radial-SKD-true-partial-pos}. We refer to the sample size $b$ as the \textit{batch size}.

Let $k \in \{1, \ldots, n\}$ and $l \in \{1, \ldots, d\}$; the partial derivative \eqref{eq:radial-SKD-true-partial-pos} can be rewritten as 
\[
 \partial_{[s_k]_l} R(\CS)
 = \frac{T_{1}^{2}}{\|\BfK_{\CS}\|_{\Frob}^{4}}\Upsilon(\CS)
- \frac{2 T_{1} T_{2}^{k,l}}{\|\BfK_\CS\|_\Frob^2},
\]
with
$T_{1} = \sum_{i=1}^{N} \sum_{j=1}^{n} K^2(s_{j,} x_{i})$ and 
$T_{2}^{k, l} = \sum_{i=1}^{N} \partial_{[s_{k}]_{l}}^{[\rml]} K^2(s_{k}, x_{i})$, and 
\[
\Upsilon(\CS)
=\partial_{[s_{k}]_{l}}^{[\rmd]} K^{2}(s_{k}, s_{k}) + 2 \sum_{\substack{j=1, \\ j \neq k}}^{n} \partial_{[s_{k}]_{l}}^{[\rml]} K^{2}(s_{k}, s_{j}).  
\] 
The terms $T_1$ and $T_2^{k, l}$ are the only terms in \eqref{eq:radial-SKD-true-partial-pos} that depend on $\CD$. From a uniform random sample $\BfX=\{X_{1},\cdots,X_{b}\}$, we define the unbiased estimators
$\hat{T}_{1}(\BfX)$ of $T_1$, and $\hat{T}_{2}^{k, l}(\BfX)$ of $T_2^{k, l}$, as 
\[
\hat{T}_{1}(\BfX) = \frac{N}{b} \sum_{i=1}^{n}\sum_{j=1}^{b}K^{2}(s_{i}, X_{j}),
\quad\text{ and }\quad
\hat{T}_{2}^{k, l}(\BfX) =\frac{N}{b} \sum_{j=1}^{b}\partial_{[s_{k}]_{l}}^{[\rml]} K^{2}(s_{k}, X_{j}).
\]
In what follows, we discuss the properties of some stochastic approximations of the gradient of $R$ that can be defined from such estimators.

\paragraph{One-Sample Approximation. }
Using a single random sample $\BfX=\{X_{1},\cdots,X_{b}\}$ of size $b$, we can define the following stochastic approximation of the partial derivative \eqref{eq:radial-SKD-true-partial-pos}:
\begin{equation}\label{eq:OneSampApp}
 \hat{\partial}_{[s_{k}]_{l}}R(\CS;\BfX)
 =\frac{\hat{T}_{1}(\BfX)^{2}}{\|\BfK_{\CS}\|_{\Frob}^{4}}\Upsilon(\CS)
-\frac{2 \hat{T}_{1}(\BfX) \hat{T}_{2}^{k, l}(\BfX)}{\|\BfK_\CS\|_\Frob^2}. 
\end{equation}
An evaluation of $\hat{\partial}_{[s_{k}]_{l}}R(\CS;\BfX)$ has complexity $\CO(n^{2}+nb)$, as opposed to $\CO(n^{2}+nN)$ for the corresponding exact partial derivative. 
However, due to the dependence between $\hat{T}_1(\BfX)$ and $\hat{T}_{2}^{k, l}(\BfX)$, and to the fact that $\hat{\partial}_{[s_{k}]_{l}}R(\CS;\BfX)$ involves the square of $\hat{T}_1(\BfX)$, the stochastic partial derivative $\hat{\partial}_{[s_{k}]_{l}}R(\CS;\BfX)$ will generally be a biased estimator of $\partial_{[s_{k}]_{l}}R(\CS)$. 

\paragraph{Two-Sample Approximation. }
To obtain an unbiased estimator of the partial derivative \eqref{eq:radial-SKD-true-partial-pos}, instead of considering a single random sample, we may define a stochastic approximation based on two independent random samples $\BfX=\{X_{1},\cdots, X_{b_{\BfX}}\}$ and $\BfY=\{Y_{1},\cdots, Y_{b_{\BfY}}\}$, consisting of $b_{\BfX}$ and $b_{\BfY}\in\BN$ copies of $X$ (i.e. consisting of uniform random variables on $\CD$), with $b=b_{\BfX}+b_{\BfY}$.
The two-sample estimator of \eqref{eq:radial-SKD-true-partial-pos} is then given by
\begin{equation}\label{eq:TwoSampApp}
 \hat{\partial}_{[s_{k}]_{l}}R(\CS;\BfX,\BfY)
 =\frac{\hat{T}_{1}(\BfX) \hat{T}_{1}(\BfY)}{\|\BfK_{\CS}\|_{\Frob}^{4}}\Upsilon(\CS)
 - \frac{2 \hat{T}_{1}(\BfX) \hat{T}_{2}^{k, l}(\BfY)}{\|\BfK_\CS\|_\Frob^2},    
\end{equation}
and since $\BE\big[\hat{T}_1(\BfX)\hat{T}_1(\BfY)\big]=T_{1}^{2} $ and  $\BE\big[\hat{T}_1(\BfX) \hat{T}_{2}^{k, l}(\BfY)\big]=T_{1}T_{2}^{k, l}$, we have
\[
\BE\Big[\hat{\partial}_{[s_{k}]_{l}}R(\CS;\BfX,\BfY)\Big]=\partial_{[s_{k}]_{l}}R(\CS).
\]

Although being unbiased, for a common batch size $b$, the variance of the two-sample estimator \eqref{eq:TwoSampApp} will generally be larger than the variance of the one-sample estimator \eqref{eq:OneSampApp}. In our numerical experiments, the larger variance of the unbiased estimator \eqref{eq:TwoSampApp} seems to actually slow down the descent when compared to the descent obtained with the one-sample estimator \eqref{eq:OneSampApp}.

\begin{remark}
 While considering two independent samples $\BfX$ and $\BfY$, the two terms $\hat{T}_1(\BfX)\hat{T}_1(\BfY)$ and $\hat{T}_{1}(\BfX) \hat{T}_{2}^{k, l}(\BfY)$ appearing in \eqref{eq:TwoSampApp} are dependent. This dependence may complicate the analysis of the properties of the resulting SGD; nevertheless, this issue might be overcome by considering four independent samples instead of two. 
\fin
\end{remark}

\section{Numerical Experiments}
\label{sec:experiments}

Throughout this section, the matrices $\BfK$ are defined from multisets
$\CD=\{x_{1},\cdots,x_{N}\}\subset\BR^{d}$ 
and from kernels $K$ of the form $K(x,t)=e^{-\rho\|x-t\|^{2}}$, with $\rho>0$ and where $\|.\|$ is the Euclidean norm of $\BR^{d}$ (Gaussian kernel). Except for the synthetic example of Section~\ref{sec:BiGauss}, all the multisets $\CD$ we consider consist of the entries of data sets available on the UCI Machine Learning Repository; see \cite{Dua2019UCI}. 

Our experiments are based on the following protocol: for a given $n\in\BN$, we consider an initial Nystr\"om sample $\CS^{(0)}$ consisting of $n$ points drawn uniformly at random, without replacement, from $\CD$. The initial sample $\CS^{(0)}$ is regarded as an element of $\SX^{n}$, and used to initialise a GD or SGD, with fixed stepsize $\gamma>0$, for the minimisation of $R$ over $\SX^{n}$, yielding, after $T\in\BN$ iterations, a locally optimised Nystr\"om sample $\CS^{(T)}$. The SGDs are performed with the one-sample estimator \eqref{eq:OneSampApp} and are based on independent and identically distributed uniform random variables on $\CD$ (i.e. i.i.d. sampling), with batch size $b\in\BN$; see Section~\ref{sec:stoch-approx}.  We assess the accuracy of the Nystr\"om approximations of $\BfK$ induced by $\CS^{(0)}$ and $\CS^{(T)}$ in terms of radial SKD and of the classical criteria \ref{it:CrtiTrace}-\ref{it:CritSpec}.

For a Nystr\"om sample $\CS\in\SX^{n}$ of size $n\in\BN$, the matrix $\hat{\BfK}(\CS)$ is of rank at most $n$.
Following \cite{gittens2016revisiting, derezinski2020improved}, 
to further assess the efficiency of the approximation of $\BfK$ induced by $\CS$, we introduce the \textit{approximation factors}
\begin{equation}\label{eq:efficiency}
\CE_{\mathrm{tr}}(\CS) = \frac{\|\BfK - \hat{\BfK}(\CS)\|_{*}}{\|\BfK - \BfK_n\|_{*}}, \quad
\CE_{\mathrm{F}}(\CS) = \frac{\|\BfK - \hat{\BfK}(\CS)\|_{\Frob}}{\|\BfK - \BfK_n\|_{\Frob}}, \quad
\text{ and }
\CE_{\mathrm{sp}}(\CS) = \frac{\|\BfK - \hat{\BfK}(\CS)\|_{2}}{\|\BfK - \BfK_{n}\|_{2}}, 
\end{equation}
where $\BfK_n$ denotes an optimal rank-$n$ approximation of $\BfK$ (i.e. the approximation of $\BfK$ obtained by truncation of a spectral expansion of $\BfK$ and based on $n$ of the largest eigenvalues of $\BfK$).  The closer $\CE_{\mathrm{tr}}(\CS)$, $\CE_{\mathrm{F}}(\CS)$ and $\CE_{\mathrm{sp}}(\CS)$ are to $1$, the more efficient the approximation is.

\subsection{Bi-Gaussian Example}
\label{sec:BiGauss}
We consider a kernel matrix $\BfK$ defined by a set $\CD$ consisting of $N=2{,}000$ points in $[-1,1]^{2}\subset\BR^{2}$ (i.e. $d=2$); for the kernel parameter, we use $\rho=1$. A graphical representation of the set $\CD$ is given in Figure~\ref{fig:BiGaussLandmarkPath};  it consists of $N$ independent realisations of a bivariate random variable whose density is proportional to the restriction of a bi-Gaussian density to the set $[-1,1]^{2}$ (the two modes of the underlying distribution are located at $(-0.8,0.8)$ and $(0.8,-0.8)$, and the covariance matrix of the each Gaussian density is $\BI_{2}/2$, with  $\BI_{2}$ the $2\times 2$ identity matrix). 

\begin{figure}[h!]
\begin{center}
\includegraphics[width=.9\linewidth]{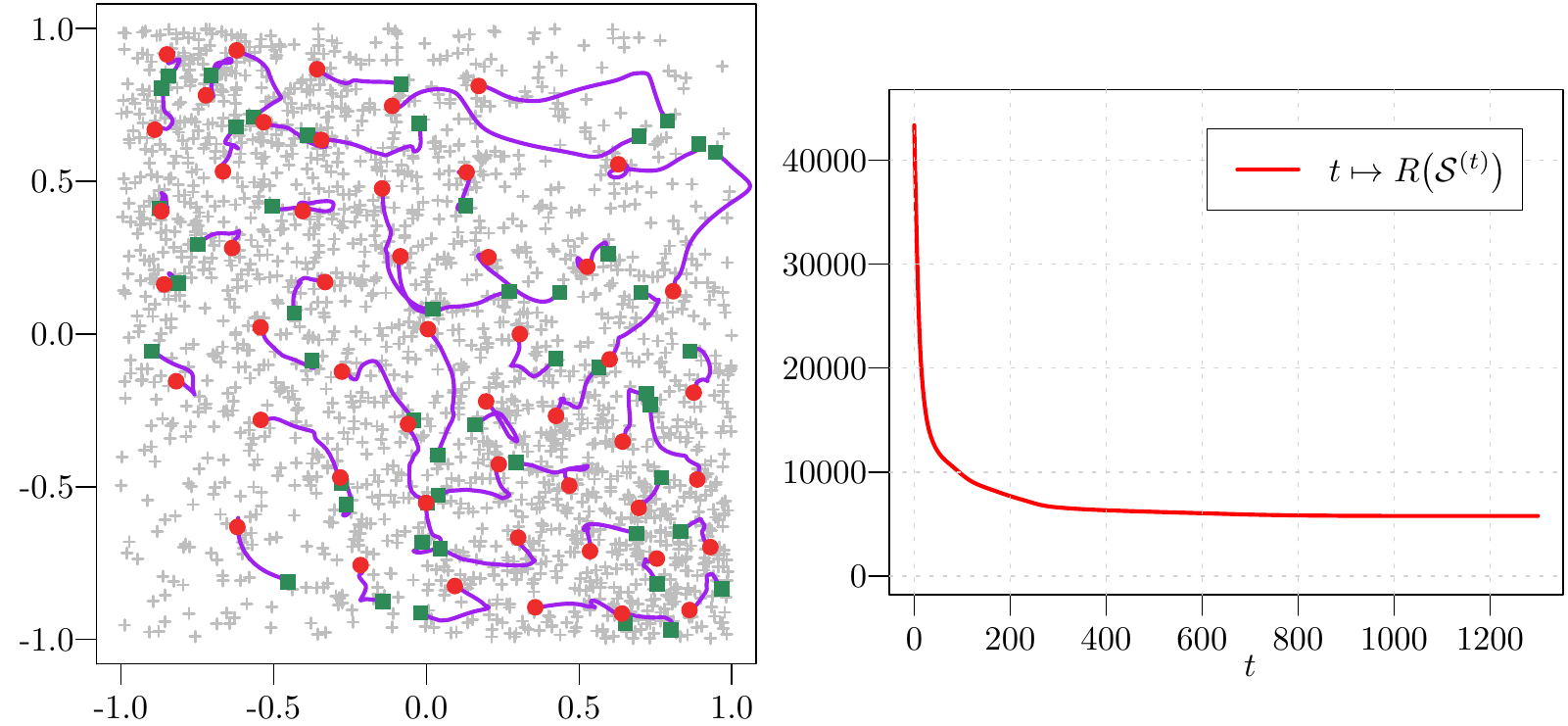}
\end{center}
\caption{Graphical representation of the path followed by the landmark points of a Nystrom sample during the local minimisation of $R$ through GD, with $n=50$, $\gamma=10^{-6}$ and $T=1{,}300$; the green squares are the landmark points of the initial sample $\CS^{(0)}$, the red dots are the landmark points of the locally optimised sample $\CS^{(T)}$, and the purple lines correspond to the paths followed by each landmark point (left). The corresponding decay of the radial SKD is also presented (right). }
\label{fig:BiGaussLandmarkPath}
\end{figure}

The initial samples $\CS^{(0)}$ are optimised via GD with stepsize $\gamma=10^{-6}$ and for a fixed number of iterations $T$. A graphical representation of the paths followed by the landmark points during the optimisation process is given in Figure~\ref{fig:BiGaussLandmarkPath} (for $n=50$ and $T=1{,}300$); we observe that the landmark points exhibit a relatively complex dynamic, some of them showing significant displacements from their initial positions. The optimised landmark points concentrate around the regions where the density of points in $\CD$ is the largest, and inherit a space-filling-type property in accordance with the stationarity of the kernel $K$.

To assess the improvement yielded by the optimisation process, for a given number of landmark points $n\in\BN$, we randomly draw an initial Nystr\"om sample $\CS^{(0)}$ from $\CD$ (uniform sampling without replacement) and compute the corresponding locally optimised sample $\CS^{(T)}$ (GD with $\gamma=10^{-6}$ and $T=1{,}000$). We then compare $R\big(\CS^{(0)}\big)$ with $R\big(\CS^{(T)}\big)$, and compute the corresponding approximation factors with respect to the trace, Frobenius and spectral norms, see \eqref{eq:efficiency}. We consider three different values of $n$, namely $n=20$, $50$ and $80$, and each time perform $m=1{,}000$ repetitions of this experiment. Our results are presented in Figure~\ref{fig:boxplotBiGauss}; we observe that, independently of $n$, the local optimisation produces a significant improvement of the Nystr\"om approximation accuracy for all the criterion considered; the improvements are particularly noticeable for the trace and Frobenius norms, and slightly less for the spectral norm (which of the three, appears the coarsest measure of the approximation accuracy). Remarkably, the efficiencies of the locally optimised Nystr\"om samples are relatively close to each other, in particular in terms of trace and Frobenius norms, suggesting that a large proportion of the local minima of the radial SKD induce approximations of comparable quality.

\begin{figure}[ht!]
\begin{center}
\includegraphics[width=\linewidth]{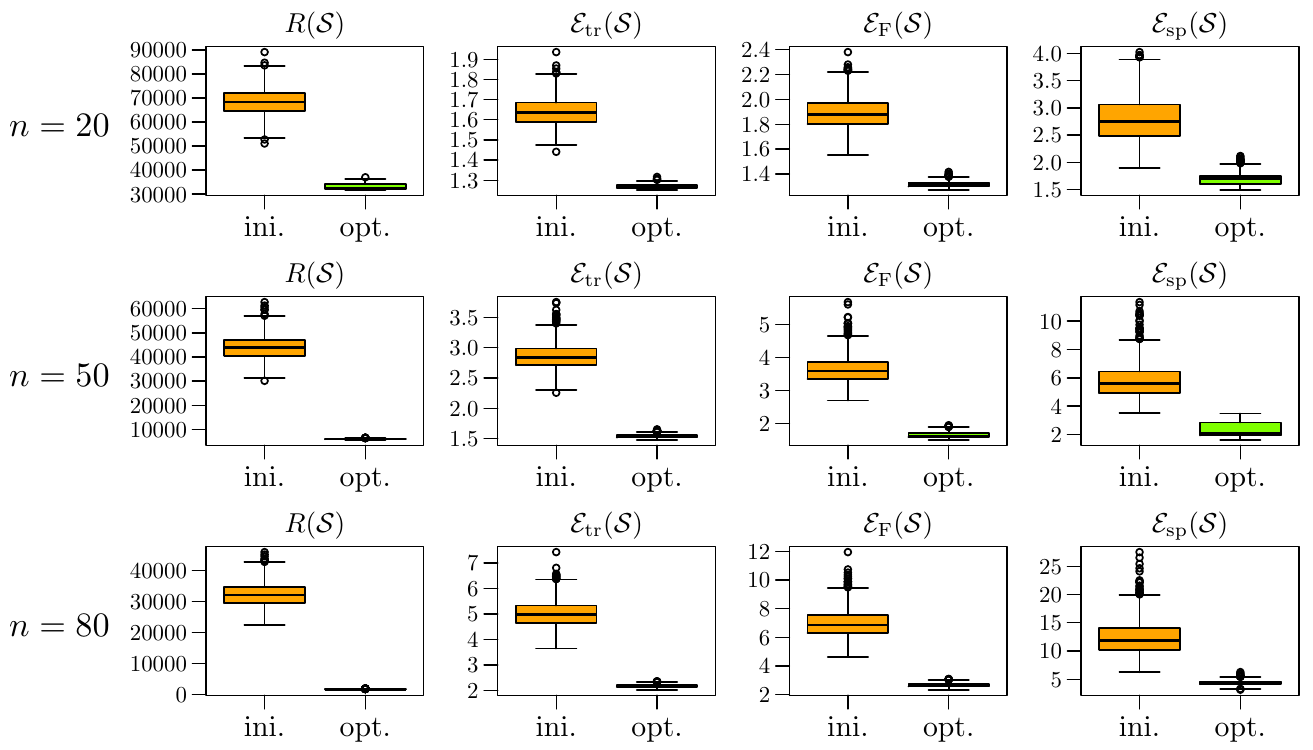}
\end{center}
\caption{For the Bi-Gaussian example, comparison of the efficiency of the Nystr\"om approximations for the initial samples $\CS^{(0)}$ and the locally optimised samples $\CS^{(T)}$ (optimisation through GD with $\gamma=10^{-6}$ and $T=1{,}000$). Each row corresponds to a given value of $n$; in each case $m=1{,}000$ repetitions are performed. The first column corresponds to the radial SKD, and the following three correspond to the approximation factors defined in \eqref{eq:efficiency}.}
\label{fig:boxplotBiGauss}
\end{figure}

\subsection{Abalone Data Set}
\label{sec:abalone}
We now consider the $d=8$ attributes of the Abalone data set. After removing two observations that are clear outliers, we are left with $N=4{,}175$ entries. Each of the $8$ features is standardised such that it has zero mean and unit variance. We set $n=50$ and
consider three different values of the kernel paramater $\rho$, namely $\rho=0.25$, $1$, and $4$; this values are chosen so that the eigenvalues of the kernel matrix $\BfK$ exhibit sharp, moderate and shallower decays, respectively. For the Nystr\"om sample optimisation, we use SGD with i.i.d. sampling and batch size $b=50$, $T=10{,}000$ and $\gamma=8\times10^{-7}$;  these values were chosen to obtain relatively efficient optimisations for the whole range of values of $\rho$ we consider. For each value of $\rho$, we perform $m=200$ repetitions. The results are presented in Figure~\ref{fig:abalone-50}. 

\begin{figure}[ht!]
\begin{center}
\includegraphics[width=\linewidth]{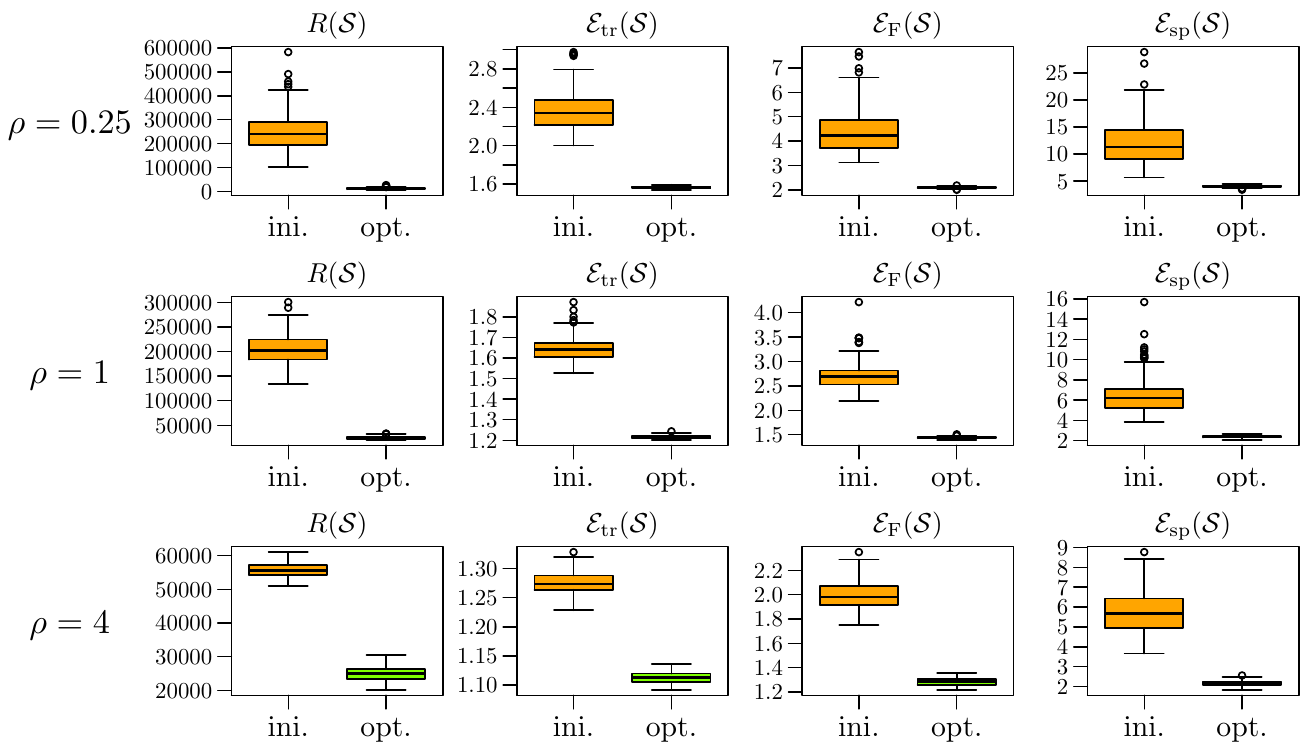}
\end{center}
\caption{For the Abalone data set with $n = 50$ and $\rho\in\{0.25, 1, 4\}$, comparison of the efficiency of the Nystr\"om approximations for the initial Nystr\"om samples $\CS^{(0)}$ and the locally optimised samples $\CS^{(T)}$ (SGD with i.i.d sampling, $b=50$, $\gamma=8\times10^{-7}$ and $T=10{,}000$). Each row corresponds to a given value of $\rho$; in each case, $m=200$ repetitions are performed. }
\label{fig:abalone-50}
\end{figure}

We observe that regardless of the values of $\rho$ and in comparison with the initial Nystr\"om samples, the efficiencies of the locally optimised samples in terms of trace, Frobenius and spectral norms are significantly improved. As observed in Section~\ref{sec:BiGauss}, the gains yielded by the local optimisations are more evident in terms of trace and Frobenius norms, and the impact of the initialisation appears limited.

\subsection{MAGIC Data Set} \label{sec:magic}

We consider the $d=10$ attributes of the MAGIC Gamma Telescope data set. In pre-processing, we remove the $115$ duplicated entries in the data set, leaving us with $N=18{,}905$ data points; we then standardise each of the $d=10$ features of the data set. For the kernel parameter, we use $\rho=0.2$. 

In Figure~\ref{fig:magic}, we present the results obtained after the local optimisation of $m=200$ random initial Nystr\"om samples of size $n=100$ and $200$. Each optimisation was performed through SGD with i.i.d. sampling, batch size $b=50$ and stepsize $\gamma=5\times 10^{-8}$; as number of iterations, for $n=100$, we used $T=3{,}000$, and $T=4{,}000$ for $n=200$. The optimisation parameters were chosen to obtain relatively efficient but not fully completed descents, as illustrated in Figure~\ref{fig:magic}. Alongside the radial SKD, we only compute the approximation factor corresponding to the trace norm (the trace norm is indeed the least costly to evaluate of the three matrix norms we consider, see Section \ref{sec:Nystrom-accuracy}). As in the previous experiments, we observe a significant improvement of the initial Nystr\"om samples obtained by local optimisation of the radial SKD.

\begin{figure}[ht!]
\begin{center}
\includegraphics[width=\linewidth]{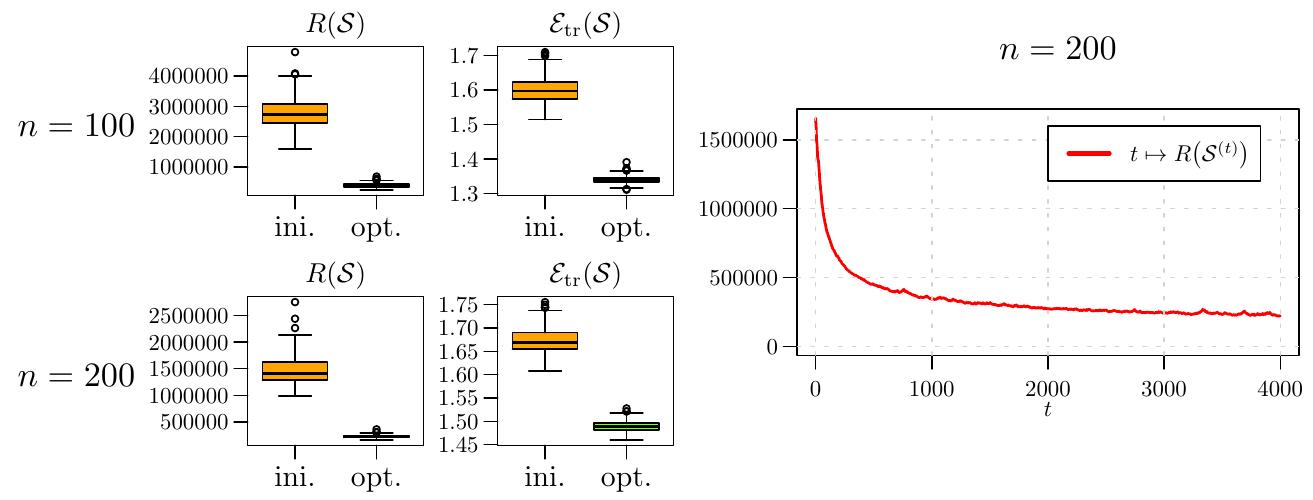}
\end{center}
\caption{For the MAGIC data set, boxplots of the radial SKD $R$ and of the approximation factor $\CE_{\mathrm{tr}}$ before and after the local optimisation via SGD of random Nystr\"om samples of size $n=100$ and $200$; for each value of $n$, $m=200$ repetitions are performed. The SGD is based on i.i.d. sampling, with $b=50$ and $\gamma=5\times 10^{-8}$; for $n=100$, the descent is stopped after $T=3{,}000$ iterations, and after $T=4{,}000$ iterations for $n=200$ (left). A graphical representation of the decay of the radial SKD is also presented for $n=200$ (right). }
\label{fig:magic}
\end{figure}

\subsection{MiniBooNE Data Set} \label{sec:miniboone}

In this last experiment, we consider the $d=50$ attributes of the MiniBooNE particle identification data set. In pre-processing, we remove the $471$ entries in the data set with missing values, and $1$ entry appearing as a clear outlier, leaving us with $N = 129{,}592$ data points; we then standardise each of the $d=50$ features of the data set. We use $\rho=0.04$ (kernel parameter).

\begin{figure}[ht!]
\begin{center}
\includegraphics[width=.9\linewidth]{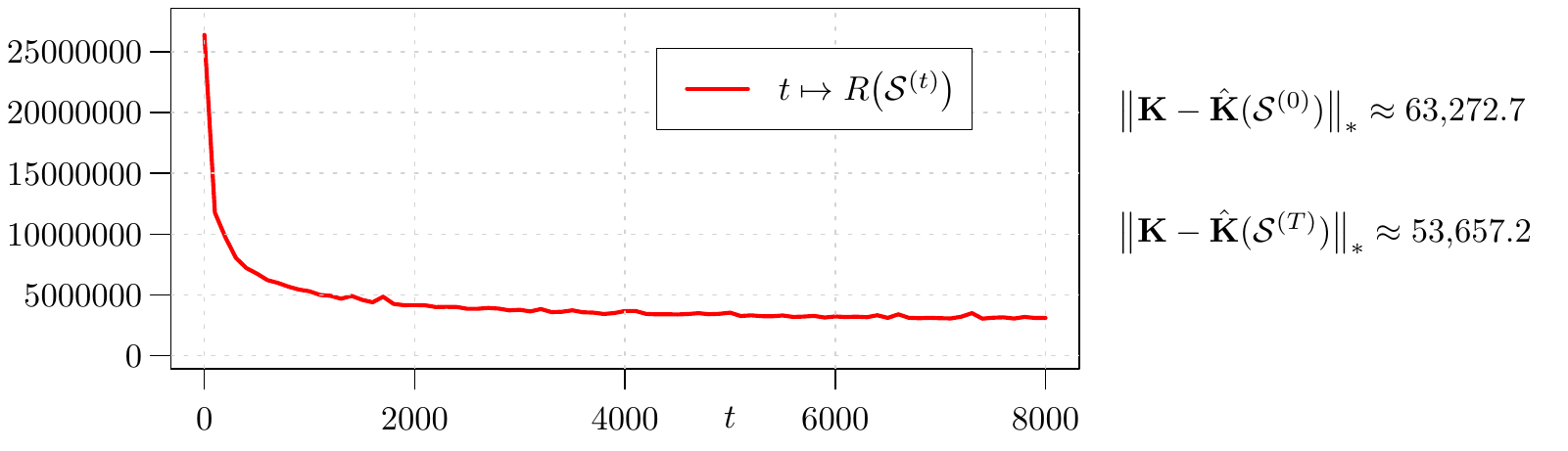}
\end{center}
\caption{For the MiniBooNE data set, decay of the radial SKD during the optimisation of a random initial Nystr\"om sample of size $n=1{,}000$. The SGD is based on i.i.d. sampling with batch size $b=200$ and stepsize $\gamma=2\times 10^{-7}$, and the descent is stopped after $T=8{,}000$ iterations; the cost is evaluated every $100$ iterations. }
\label{fig:miniboone}
\end{figure}

We consider a random initial Nystr\"om sample of size $n=1{,}000$, and optimise it through SGD with i.i.d. sampling, batch size $b=200$, stepsize $\gamma=2\times 10^{-7}$; the descent is stopped after $T=8{,}000$ iterations. The resulting decay of the radial SKD is presented in Figure~\ref{fig:miniboone} (the cost is evaluated every $100$ iterations), and the trace norm of the Nystr\"om approximation error for the initial and locally optimised samples are reported. In terms of computation time, on our machine (endowed with an 3.5 GHz Dual-Core Intel Core i7 processor, and using a single-threaded C implementation interfaced with R), for $n=1{,}000$, an evaluation of the radial SKD (up to the constant $\|\BfK\|_{\Frob}^{2}$) takes $6.8~\seco$, while an evaluation of the term $\|\BfK-\hat{\BfK}(S)\|_{*}$ takes $6{,}600~\seco$; performing the optimisation reported in Figure~\ref{fig:miniboone} without checking the decay of the cost takes $1{,}350~\seco$.
This experiment illustrates the ability of the considered framework to tackle relatively large problems. 

\section{Conclusion}
\label{sec:conclusion}

We demonstrated the relevance of the radial-SKD-based framework for the local optimisation, through SGD, of Nystr\"om samples for SPSD kernel-matrix approximation. We studied the Lipschitz continuity of the underlying gradient and discussed its stochastic approximation. We performed numerical experiments illustrating that local optimisation of the radial SKD yields significant improvement of the Nystr\"om approximation in terms of trace, Frobenius and spectral norms.

In our experiments, we implemented SGD with i.i.d. sampling, fixed stepsize and fixed number of iterations; although already bringing satisfactory results, to improve the time efficiency of the approach, the optimisation strategy could be accelerated by considering for instance adaptive stepsize, parallelisation or momentum-type techniques (see \cite{ruder2016overview} for an overview). The initial Nystr\"om samples $\CS^{(0)}$ we considered were draw uniformly at random without replacement; while our experiments suggest that the local minima of the radial SKD often induce approximations of comparable quality, the use of more efficient initialisation strategies may be investigated (see e.g. \cite{niederreiter1992random, kumar2012sampling, wang2016towards, cai2018smash, derezinski2020improved}).

As a side note, when considering the trace norm, the Nystr\"om sampling problem is intrinsically related to the
\emph{integrated-mean-squared-error} design criterion in kernel regression (see e.g. \cite{rasmussen2006gaussian, Gauthier2017IMSE, santner2018design}); consequently the approach considered in this paper may be used for the design of experiments for such models.

\appendix
\section*{Appendix}

\begin{proof}[Proof of Theorem \ref{thm:lipschitz}]
We consider a Nystr\"om sample $\CS \in \SX^{n}$ and introduce 
\begin{equation} \label{eq:c-S}
	c_{\CS} = \frac{1}{\| \BfK_{\CS} \|_{\Frob}^{2}} \sum_{i=1}^{N} \sum_{j=1}^{n} K^{2}(x_{i}, s_{j}). 
\end{equation}
In view of \eqref{eq:radial-SKD-true-partial-pos}, the partial derivative of $R$ at $\CS$ with respect to the $l$-th coordinate of the $k$-th landmark point $s_{k}$ can be written as
\begin{equation} \label{eq:radial-SKD-true-partial-pos-c-S}
\partial_{[s_{k}]_{l}} R(\CS) = c_{\CS}^{2} \bigg( \partial_{[s_{k}]_{l}}^{[\rmd]} K^{2}(s_{k}, s_{k}) + 2 \sum_{\substack{j=1, \\ j \neq k}}^{n} \partial_{[s_{k}]_{l}}^{[\rml]} K^{2}(s_{k}, s_{j}) \bigg) - 2 c_{\CS} \sum_{i=1}^{N} \partial_{[s_{k}]_{l}}^{[\rml]} K^{2}(s_{k}, x_{i}).
\end{equation}
For $k$ and $k' \in \{1, \cdots, n\}$ with $k \neq k'$, and for $l$ and $l' \in \{1, \cdots, d\}$, the second-order partial derivatives of $R$ at $\CS$, with respect to the coordinates of the landmark points in $\CS$, verify
\begin{align}
\begin{split} \label{eq:radial-SKD-partial-2-same}
	\partial_{[s_{k}]_{l}} \partial_{[s_{k}]_{l'}} R(\CS)
		&= c_{\CS}^{2} \partial_{[s_{k}]_{l}}^{[\rmd]} \partial_{[s_{k}]_{l'}}^{[\rmd]} K^{2}(s_{k}, s_{k}) + 2 c_{\CS} (\partial_{[s_{k}]_{l'}} c_{\CS}) \partial_{[s_{k}]_{l}}^{[\rmd]} K^{2}(s_{k}, s_{k}) \\
		&\quad + 2 c_{\CS}^{2} \sum_{\substack{j=1, \\ j \neq k}}^{n} \partial_{[s_{k}]_{l}}^{[\rml]} \partial_{[s_{k}]_{l'}}^{[\rml]} K^{2}(s_{k}, s_{j}) + 4 c_{\CS} (\partial_{[s_{k}]_{l'}} c_{\CS}) \sum_{\substack{j=1, \\ j \neq k}}^{n} \partial_{[s_{k}]_{l}}^{[\rml]} K^{2}(s_{k}, s_{j}) \\
		&\quad - 2 c_{\CS} \sum_{i=1}^{N} \partial_{[s_{k}]_{l}}^{[\rml]} \partial_{[s_{k}]_{l'}}^{[\rml]} K^{2}(s_{k}, x_{i}) - 2 (\partial_{[s_{k}]_{l'}} c_{\CS}) \sum_{i=1}^{N} \partial_{[s_{k}]_{l}}^{[\rml]} K^{2}(s_{k}, x_{i}), \text{ and}
\end{split}
\end{align}
\begin{align}
\begin{split} \label{eq:radial-SKD-partial-2-diff}
	\partial_{[s_{k}]_{l}} \partial_{[s_{k'}]_{l'}}& R(\CS)
		= 2 c_{\CS} (\partial_{[s_{k'}]_{l'}} c_{\CS}) \partial_{[s_{k}]_{l}}^{[\rmd]} K^{2}(s_{k}, s_{k}) + 2 c_{\CS}^{2} \partial_{[s_{k}]_{l}}^{[\rml]} \partial_{[s_{k'}]_{l'}}^{[\rmr]} K^{2}(s_{k}, s_{k'}) \\ 
		&\quad + 4 c_{\CS} (\partial_{[s_{k'}]_{l'}} c_{\CS}) \sum_{\substack{j=1, \\ j \neq k}}^{n} \partial_{[s_{k}]_{l}}^{[\rml]} K^{2}(s_{k}, s_{j}) - 2(\partial_{[s_{k'}]_{l'}} c_{\CS}) \sum_{i=1}^{N} \partial_{[s_{k}]_{l}}^{[\rml]} K^{2}(s_{k}, x_{i}), 
\end{split}
\end{align}
where the partial derivative of $c_{\CS}$ with respect to the $l$-th coordinate of the $k$-th landmark point $s_{k}$ is given by
\begin{equation} \label{eq:c-S-partial}
	\partial_{[s_{k}]_{l}} c_{\CS} = \frac{1}{\| \BfK_{\CS} \|_{\Frob}^{2}} \bigg(\sum_{i=1}^{N} \partial_{[s_{k}]_{l}}^{[\rml]} K^{2}(s_{k}, x_{i}) - c_{\CS} \partial_{[s_{k}]_{l}}^{[\rmd]} K^{2}(s_{k}, s_{k}) - 2 c_{\CS} \sum_{\substack{j = 1, \\ j \neq k}}^{n} \partial_{[s_{k}]_{l}}^{[\rml]} K^{2}(s_{k}, s_{j})\bigg).
\end{equation}
From \ref{itm:sq-kernel-diag-lower-bound}, we have 
\begin{equation} \label{eq:frob-lower-bound}
	\| \BfK_{\CS} \|_{\Frob}^{2} = \sum_{i=1}^{n} \sum_{j=1}^{n} K^{2}(s_{i}, s_{j}) \geqslant \sum_{i=1}^{n} K^{2}(s_{i}, s_{i}) \geqslant n \alpha.
\end{equation}

By the Schur product theorem, the squared kernel $K^{2}$ is SPSD;  
we denote by $\CG$ the RKHS of real-valued functions on $\SX$ for which $K^{2}$ is reproducing. For $x$ and $y\in\SX$, we have $K^{2}(x, y) = \langle k_{x}^{2}, k_{y}^{2} \rangle_{\CG}$, with $\langle \cdot, \cdot \rangle_{\CG}$ the inner product on $\CG$, and where $k_{x}^{2}\in\CG$ is such that $k_{x}^{2}(t)=K^{2}(t,x)$, for all $t\in\SX$. From the Cauchy-Schwartz inequality, we have
\begin{align}
\sum_{i=1}^{N} \sum_{j=1}^{n} K^{2}(s_{j}, x_{i})
&= \sum_{i=1}^{N} \sum_{j=1}^{n} \langle k_{s_{j}}^{2}, k_{x_{i}}^{2} \rangle_{\CG}
= \bigg\langle \sum_{j=1}^{n} k_{s_{j}}^{2}, \sum_{i=1}^{N} k_{x_{i}}^{2} \bigg\rangle_{\CG} \nonumber \\
&\leqslant \bigg\| \sum_{j=1}^{n} k_{s_{j}}^{2} \bigg\|_{\CG} \bigg\| \sum_{i=1}^{N} k_{x_{i}}^{2} \bigg\|_{\CG}
= \| \BfK_{\CS} \|_{\Frob} \| \BfK \|_{\Frob}. \label{eq:frob-CS}
\end{align}
By combining \eqref{eq:c-S} with inequalities \eqref{eq:frob-lower-bound} and \eqref{eq:frob-CS}, we obtain
\begin{equation} \label{eq:c-S-bound}
	0 \leqslant c_{\CS} \leqslant \frac{\| \BfK \|_{\Frob}}{\| \BfK_{\CS} \|_{\Frob}} \leqslant \frac{\| \BfK \|_{\Frob}}{\sqrt{n \alpha}} = C_{0}.
\end{equation}
Let $k \in \{1, \ldots, n\}$ and let $l \in \{1, \ldots, d\}$. From equation \eqref{eq:c-S-partial}, and using inequalities \eqref{eq:frob-lower-bound} and \eqref{eq:c-S-bound} together with \ref{itm:sq-kernel-partial-bound}, we obtain
\begin{equation} \label{eq:c-S-partial-bound}
	|\partial_{[s_{k}]_{l}} c_{\CS}| \leqslant \frac{M_{1}}{n \alpha} [N + (2n - 1)C_{0}] = C_{1}.
\end{equation}
In addition, let $k' \in \{1, \ldots, n\} \setminus \{k\}$ and $l' \in \{1, \ldots, d\}$; from equations \eqref{eq:radial-SKD-partial-2-same}, \eqref{eq:radial-SKD-partial-2-diff}, \eqref{eq:c-S-bound} and \eqref{eq:c-S-partial-bound}, and conditions \ref{itm:sq-kernel-partial-bound} and \ref{itm:sq-kernel-partial-2-bound}, we get
\begin{align}
&|\partial_{[s_{k}]_{l}} \partial_{[s_{k}]_{l'}} R(\CS)| \nonumber \\
&\quad \leqslant C_{0}^{2} M_{2} + 2 C_{0} C_{1} M_{1} + 2(n - 1)C_{0}^{2} M_{2} + 4(n - 1)C_{0} C_{1} M_{1} + 2 C_{0} M_{2} N + 2 C_{1} M_{1} N \nonumber \\
&\quad = (2n - 1)C_{0}^{2} M_{2} + (4n - 2)C_{0} C_{1} M_{1} + 2N(C_{0} M_{2} + C_{1} M_{1}), \label{eq:radial-SKD-partial-2-same-bound}
\end{align}
and
\begin{align}
|\partial_{[s_{k}]_{l}} \partial_{[s_{k'}]_{l'}} R(\CS)|
&\leqslant 2 C_{0} C_{1} M_{1} + 2 C_{0}^{2} M_{2} + 4(n - 1)C_{0} C_{1} M_{1} + 2 C_{1} M_{1} N \nonumber \\
&= 2 C_{0}^{2} M_{2} + (4n - 2)C_{0} C_{1} M_{1} + 2 N C_{1} M_{1}. \label{eq:radial-SKD-partial-2-diff-bound}
\end{align}
For $k, k' \in \{1, \ldots, n\}$, we denote by $\BfB^{k, k'}$ the $d \times d$ matrix with $l, l'$ entry given by \eqref{eq:radial-SKD-partial-2-same} if $k=k'$, and by \eqref{eq:radial-SKD-partial-2-diff} otherwise. The Hessian $\nabla^{2} R(\CS)$ can then be represented as a block-matrix, that is

\begin{equation*}
\nabla^{2} R(\CS) =
\begin{bmatrix}
\BfB^{1, 1}	& \cdots	& \BfB^{1, n} \\
\vdots		& \ddots	& \vdots \\
\BfB^{n, 1} & \cdots	& \BfB^{n, n}
\end{bmatrix}
\in \BR^{nd\times nd}. 
\end{equation*}
The $d^{2}$ entries of  the $n$ diagonal blocks of $\nabla^{2} R(\CS)$ are of the form \eqref{eq:radial-SKD-partial-2-same}, and the $d^{2}$ entries of the $n(n-1)$ off-diagonal blocks of $\nabla^{2} R(\CS)$ are the form \eqref{eq:radial-SKD-partial-2-diff}. From inequalities \eqref{eq:radial-SKD-partial-2-same-bound} and \eqref{eq:radial-SKD-partial-2-diff-bound}, we obtain
\begin{equation*}
\| \nabla^{2} R(\CS) \|_{2}^{2}
\leqslant \| \nabla^{2} R(\CS) \|_{\Frob}^{2}
= \sum_{k=1}^{n} \sum_{l=1}^{d} \sum_{l'=1}^{d} [\BfB^{k, k}]_{l, l'}^{2} + \sum_{k=1}^{n} \sum_{\substack{k'=1, \\ k' \neq k}}^{n} \sum_{l=1}^{d} \sum_{l'=1}^{d} [\BfB^{k, k'}]_{l, l'}^{2}
\leqslant L^{2}, 
\end{equation*}
with 
\begin{align*}
\begin{split}
L&=\big(nd^{2}[(2n - 1)C_{0}^{2} M_{2} + (4n - 2)C_{0} C_{1} M_{1} + 2N(C_{0} M_{2} + C_{1} M_{1})]^{2} \\
&\quad \phantom{x} + 4n(n - 1)d^{2} [C_{0}^{2} M_{2} + (2n - 1)C_{0} C_{1} M_{1} + N C_{1} M_{1}]^{2}\big)^{\frac{1}{2}}. 
\end{split}
\end{align*}
For all $\CS\in\SX^{n}$, the constant $L$ is an upper bound for the spectral norm of the Hessian matrix $\nabla^{2} R(\CS)$, so the gradient of $R$ is Lipschitz continuous over $\SX^{n}$, with Lipschitz constant $L$.
\end{proof}

\bibliographystyle{plain}
\bibliography{Radial_Flow}

\end{document}